\pdfoutput=1
\documentclass[letterpaper]{article} 
\usepackage[submission]{aaai23}  
\usepackage{times}  
\usepackage{helvet}  
\usepackage{courier}  
\usepackage[hyphens]{url}  
\usepackage{graphicx} 
\usepackage{natbib}  
\usepackage{caption} 
\usepackage{algorithm}
\usepackage{algorithmic}

\usepackage{amsmath}
\usepackage{amssymb}
\usepackage{mathtools}
\usepackage{amsthm}
\usepackage{enumerate}
\usepackage{enumitem}

\usepackage[utf8]{inputenc} 
\usepackage{booktabs}       
\usepackage{amsfonts}       
\usepackage{nicefrac}       
\usepackage{microtype}      
\usepackage{xcolor}         

\usepackage{amsmath,amsfonts,bm, amsthm, mathrsfs, mathtools}

\newtheorem{theorem}{Theorem}
\newtheorem{assumption}{Assumption}

\def\eqref#1{equation~\ref{#1}}

\def\1{\bm{1}}

\def\eps{{\epsilon}}

\DeclareMathAlphabet{\mathsfit}{\encodingdefault}{\sfdefault}{m}{sl}
\SetMathAlphabet{\mathsfit}{bold}{\encodingdefault}{\sfdefault}{bx}{n}

\def\gM{{\mathcal{M}}}

\def\sR{{\mathbb{R}}}

\usepackage{dsfont}
\def\1{\mathds{1}}

\DeclareMathOperator*{\argmax}{arg\,max}

\usepackage{newfloat}
\usepackage{listings}
\DeclareCaptionStyle{ruled}{labelfont=normalfont,labelsep=colon,strut=off} 
\floatstyle{ruled}
\newfloat{listing}{tb}{lst}{}
\floatname{listing}{Listing}
\title{Local Environment Poisoning Attacks on Federated Reinforcement Learning}
\author{
    Written by AAAI Press Staff\textsuperscript{\rm 1}\thanks{With help from the AAAI Publications Committee.}\\
    AAAI Style Contributions by Pater Patel Schneider,
    Sunil Issar,\\
    J. Scott Penberthy,
    George Ferguson,
    Hans Guesgen,
    Francisco Cruz\equalcontrib,
    Marc Pujol-Gonzalez\equalcontrib
}
\affiliations{
    \textsuperscript{\rm 1}Association for the Advancement of Artificial Intelligence\\

    1900 Embarcadero Road, Suite 101\\
    Palo Alto, California 94303-3310 USA\\
    publications23@aaai.org
}

\usepackage{bibentry}

\begin{document}

\maketitle

\begin{abstract}
Federated learning (FL) has become a popular tool for solving traditional Reinforcement Learning (RL) tasks. The multi-agent structure addresses the major concern of data-hungry in traditional RL, while the federated mechanism protects the data privacy of individual agents. However, the federated mechanism also exposes the system to poisoning by malicious agents that can mislead the trained policy. Despite the advantage brought by FL, the vulnerability of Federated Reinforcement Learning (FRL) has not been well-studied before. In this work, we propose a general framework to characterize FRL poisoning as an optimization problem and design a poisoning protocol that can be applied to policy-based FRL. Our framework can also be extended to FRL with actor-critic as a local RL algorithm by training a pair of private and public critics. We provably show that our method can strictly hurt the global objective. We verify our poisoning effectiveness by conducting extensive experiments targeting mainstream RL algorithms and over various RL OpenAI Gym environments covering a wide range of difficulty levels. Within these experiments, we compare clean and baseline poisoning methods against our proposed framework. The results show that the proposed framework is successful in poisoning FRL systems and reducing performance across various environments and does so more effectively than baseline methods. Our work provides new insights into the vulnerability of FL in RL training and poses new challenges for designing robust FRL algorithms.
\end{abstract}

\section{Introduction}

In recent years, \textit{Reinforcement Learning} (RL) has gained popularity as a paradigm for solving complex sequential decision-making problems and has been applied to a wide range of real-world problems, including game playing \citet{silver2016mastering,vinyals2019grandmaster}, autonomous driving \citet{yurtsever2020survey}, and network security \citet{xiao2018security}. In RL, the agent's goal is to learn an optimal policy that maximizes the long-term cumulative rewards, which is done by repeatedly interacting with a stochastic environment, taking actions, and receiving feedback in the form of rewards. However, despite the impressive performance of RL algorithms, they are notoriously known to be data-hungry, often suffering from poor sample efficiency \citet{dulac2021challenges,schwarzer2020data}. One traditional solution to this challenge is \textit{Parallel RL}~\citet{kretchmar2002parallel}, which adopts multiple parallel RL agents that sample data from the environment and share it with a central server, as seen in practical implementations such as game-playing \citet{mnih2016asynchronous,berner2019dota}. However, transferring raw data may not be feasible: on the one hand, it can cause significant communication costs in applications, such as the Internet of Things (IoT) \citet{wang2020federated}; On the other hand, it is not suitable for privacy-sensitive industries, such as clinical decision support \citet{liu2020reinforcement}.

In order to address the limitations of traditional parallel RL, the concept of \textit{Federated Reinforcement Learning} (FRL)\citet{fan2021fault,khodadadian2022federated} has been proposed, inspired by the recent success of \textit{Federated Learning} (FL). FRL allows multiple agents to solve the same RL task collaboratively without sharing their sensitive raw data, thus addressing the drawbacks of heavy overhead and privacy violation in traditional Parallel RL. On the communication side, the communication efficiency of FRL is improved by the ability of FL to perform multiple local model updates during each communication round. On the privacy side, FRL enables data privacy protection by only communicating model updates, but not raw data, to a central server. With advantages of addressing communication efficiency and protecting privacy, FRL is practically appealing to a wide range of applications, including IoT \citet{wang2020federated}, autonomous driving \citet{liang2023federated}, robotics \citet{liu2019lifelong}, etc.

Despite the promising advantages brought by FL into RL, the \textit{security risk} is inherited from FL as well, since the lack of transparency of local training in FL naturally causes the vulnerability of the FRL system when exposed to adversarial attacks of malicious agents. While there is an extensive line of study on attacks in FL\citet{fang2020local,bagdasaryan2020backdoor,bhagoji2019analyzing,so2020byzantine}, it has yet to be explored in the context of FRL. In this work, our goal is to \textit{systematically study the vulnerability of FRL systems when facing adversarial attacks that attempt to mislead the trained policy}.

To address the vulnerability of FRL to adversarial attacks, we start by proposing a theoretical framework that characterizes the problem of environment poisoning in FRL as an optimization problem. The attacker is assumed to have control over a few corrupted agents and can manipulate their local environments by perturbing their observations during local policy training and updates. However, the attacker does not have prior knowledge of the underlying Markov Decision Process (MDP) of the environment and can only learn it through the agents' observations. This type of attack is practical and can be easily implemented in real-world scenarios, such as buying an IoT device participating in an FRL system, and providing false signals to its sensors. To assess this risk, we design a novel poisoning mechanism that targets FRL with policy-based local training. We first design the poisoning for the general case and extend it to the Actor-Critic setting in which the attacker trains a set of public and private critics to manipulate the coordinator's model during the training process. Our method is evaluated through extensive experiments on OpenGYM environments with various difficulty levels (CartPole, InvertedPendulum, LunarLander, Hopper, Walker2d, and HalfCheetah) using different policy-gradient RL models. The results demonstrate that a few malicious agents can successfully poison the entire FRL system. Furthermore, we affirm the effectiveness of our attack model compared to baseline attacks. Our findings highlight the potential risks of FRL under adversarial attacks and inspire future research toward developing more robust algorithms.

\medskip
\noindent
\textbf{Contributions.} Our results show that the FRL framework is vulnerable to local poisoning attacks. Concretely,
\begin{itemize}
    \item We propose a novel poisoning protocol designed specifically for policy-based FRL and extend it to the Actor-Critic case, where the attacker trains public and private critics to manipulate the objective of the global model broadcast by the coordinator at each communication round.
    \item We provably show that our method can sabotage the system by degrading the global objective.
    \item We verify the effectiveness of our poisoning protocol by conducting extensive experiments targeting mainstream RL algorithms, including VPG and PPO, across various OpenGYM environments with varying difficulty levels, as discussed in Section \ref{sec:Eval}. This includes comparison with baseline and targeted attacks.
\end{itemize}
In summary, our work indicates that when FL is applied to RL training, there is a risk of exposing the system to malicious agents, making it susceptible to poisoning and threats in applications. Consequently, the trained policy of FRL may not be entirely reliable and requires a more robust and secure protocol.  

\medskip
\noindent
\textbf{Related work.}
\textit{Adversarial Attack in FL} has been analyzed in different settings  \citet{fang2020local,bagdasaryan2020backdoor,bhagoji2019analyzing,so2020byzantine}. Specifically, \citet{tolpegin2020data} studies data poisoning that compromises the integrity of the training dataset while \citet{bhagoji2019analyzing} and \citet{fang2020local} study manipulation directly on the local model parameters submitted by corrupted agents. However, these studies are under the context of Supervised Learning (SL), which is substantially different from RL in the sense of the availability of future data, the knowledge of the dynamics of the environment, etc~\citet{sun2020vulnerability}. Therefore, the existing works on FL poisoning are not directly applicable to FRL, while we focus on the poisoning designed specifically for FRL, taking into account the unique characteristics of RL.

\textit{Poisoning in RL} refers to committing attacks during the training process~\cite {zhang2020adaptive}. Poisoning attacks in RL can be categorized into two types: weak attacks that only poison data (e.g., rewards and states) and strong attacks that can manipulate actions in addition to data \citet{panagiota2020trojdrl}. In this study, we focus on weak attacks, also known as environmental poisoning, as they allow easier access for the attacker, therefore, are more likely to occur in real-world scenarios. \citet{rakhsha2020policy} formulated optimization frameworks to characterize RL poisoning, but they have limitations, such as requiring knowledge of the underlying MDP and focusing on targeted poisoning. Our proposed framework, however, considers both targeted and untargeted poisoning under the realistic assumption that the attacker does not have access to the MDP dynamics. \citet{sun2020vulnerability} designs a vulnerability-aware poisoning method for RL. Their algorithm focuses on manipulating the actor model, which cannot be directly applied to FRL as the critic model is also communicated among agents and the server. In contrast, our proposed poisoning mechanism is specifically designed for FRL and focuses on manipulating the critic model by training a pair of public and private critics.

\textit{Federated Reinforcement Learning (FRL)} has become a promising topic in recent years and inspired a series of works both on the application side~\citet{liu2019lifelong,liang2023federated,wang2020federated} and theoretical side~\citet{fan2021fault,khodadadian2022federated}. \citet{khodadadian2022federated} proves a linear convergence rate under Markovian Sampling. \citet{fan2021fault} theoretically analyzes the fault-tolerance guarantee for multi-task FRL, and \citet{anwar2021multi} examines the performance of common RL adversaries applied to FRL. Although these works study adversarial attacks, they do not allow multiple local updates, which will increase the communication costs severely in applications. In contrast, our poisoning method is designed for FRL taking into account multiple steps of local training.

\section{Preliminaries and Notations} 
\label{sec:prelim}
In this section, we overview some background and notations that are necessary for introducing the concept of poisoning in FRL. We consider single-task FRL, where a number of agents work together to achieve a common task. As such, all agents are trained on the same MDP. We consider the ubiquitous on-policy training setting \citep{singh2000convergence}.

\smallskip
\noindent
\textbf{MDP and RL.} 
An MDP is a discrete stochastic control process for decision-making~\citep{puterman1990markov} that is defined by a tuple $M = (S,A,r,P,\gamma)$, where $S$ is a discrete state space, $A$ is a discrete action space, $r(\cdot):S\times A \rightarrow \mathbb{R}$ is a reward function, $P(\cdot):S\times A \times S \rightarrow [0,1]$ is a state transition probability function, and $\gamma \in (0,1)$ is a discount factor. Given an MDP, the goal of RL is to find an optimal policy, $\pi(\cdot) : S \rightarrow \Delta_A$, where $\Delta_A$ is the set of all probability distributions over the action space $A$, which maximizes the expected accumulated discounted reward. As is common in the literature \citep{agarwal2021theory}, we often represent a policy $\pi$ by its parametrization $\theta$ (e.g., tabular parametrization or neural network weight parametrization). During the process, at each step $t$, the decision maker begins in some state $s_t$, selects an action $a_t$ according to the policy $\pi(s_t)$, receives a reward $r(s_t,a_t)$, and transitions to the next state $s_{t+1}$ with probability $P(s_{t+1}|s_t,a_t)$. This decision-making and interaction process continues until the MDP terminates.




\noindent
\textbf{Federated Reinforcement Learning (FRL)}. An FRL system consists of $n$ agents and a central server. The agents perform local training for $L$ steps and then send their updated policies to the central server. The server performs aggregation to create a central policy, which is then broadcast back to all the agents. This process is repeated for $T$ rounds, and the broadcast policy is used to initialize the next round of local training. More specifically, during each round $p\leq T$, at the \textit{start} of local step $q\leq L$, denote the policy of each agent $(i)$ by its policy parameter $\theta_{(i)}^{p,q-1}$. The agent interacts with the environment according to its current policy $\theta_{(i)}^{p,q-1}$, and collects an observation sequence $\mathcal{O}_{(i)}^{p,q}$, which is composed of sequences of states, actions, and rewards: $\mathcal{O}_{(i)}^{p,q} = (\mathcal{O}_{(i)}^{s,p,q},\mathcal{O}_{(i)}^{a,p,q},\mathcal{O}_{(i)}^{r,p,q}) = ((s_1,s_2,\ldots),(a_1,a_2,\ldots), (r_1, r_2,\ldots))$. The agent then updates its policy parameters based on its observation sequence $\mathcal{O}_{(i)}^{p,q}$. The policy update is typically given by $\theta_{(i)}^{p,q} = \arg\max_{\theta}J(\theta,\theta_{(i)}^{p,q-1},\mathcal{O}_{(i)}^{p,q})$, where $J$ is some objective function. After completing $L$ steps of local training, each agent $(i)$ updates its local policy to $\theta_{(i)}^{p,L}$ and sends it to the server.\footnote{We distinguish the parameters related to the server by index $0$.} The server then forms a new global policy $\theta_{(0)}^{p}$, which is formulated as $\theta_{(0)}^p:=\mathcal{A}^{agg} (\theta_{(0)}^{p-1},\{\theta_{(i)}^{p,L}\}_{i=1}^n)$, where $\theta_{(0)}^{p-1}$ is the server's previous global model at the end of round $p-1$, and $\mathcal{A}^{agg}$ is an aggregation algorithm. The server then broadcasts the updated global policy $\theta_{(0)}^p$ to all agents, after which each agent $i$ updates its local policy, and the system proceeds to the next round $p+1$. The local policy update is in the form of 
$\theta_{(i)}^{p+1,0} = \mathcal{A}^{up}(\theta_{(i)}^{p,L}, \theta_{(0)}^{p})$, where $\mathcal{A}^{up}$ denotes a local update function.

\section{Problem Formulation}
\label{sec:PF}
 
\subsection{Local Environment Poisoning}
\label{susec:EP}
We consider a threat model to FRL systems called local environment poisoning, where we assume that the attacker can control a small number of corrupted agents and poison their local environments, i.e., perturbing their observation during local policy training/updates. More precisely, let us denote the set of corrupted agents that the attacker has control over them by $\mathcal{M} \subset[n]$, where $[n]=\{1,2,\ldots,n\}$. Then for any agent $j \in \mathcal{M}$, at each round $p$ and step $q$, the attacker may manipulate the observation of $\mathcal{O}_{(j)}^{p,q}$ to become $\widehat{\mathcal{O}}_{(j)}^{p,q}$ by possibly targeting the states, actions, or rewards. For instance, when the attacker targets rewards, we have $\widehat{\mathcal{O}}_{(n)}^{p,q} = (\mathcal{O}_{(n)}^{s,p,q}, \mathcal{O}_{(n)}^{a,p,q}, \widehat{\mathcal{O}}_{(n)}^{r,p,q})$. Throughout this paper, we will use the notation $\ \widehat{\cdot}\ $ to indicate variables that are poisoned by the attacker.


\subsection{A General Formulation}
\label{subsec:PF1}
We propose a theoretical framework to formulate the problem of environment poisoning attacks to FRL as an optimization problem with a limited budget. The problem of poisoning FRL is formulated as a sequential bi-level optimization in Problem (\ref{opt-P}), where the attacker's objective is to manipulate the learned policy by local environment poisoning attacks. The optimization problem is defined in terms of the notations and concepts introduced in Section \ref{sec:prelim}, such as the local policies, global policy, aggregation, and update algorithms.

\smallskip
\noindent
\textbf{Attacker's objective.} 
In the optimization problem (\ref{opt-P}), the objective $\mathcal{L}_A$ represents the loss of the attacker, 
which can characterize both untargeted and targeted poisoning settings. 
In the case of untargeted poisoning, 
$\mathcal{L}_A$ is the benefit of the server, 
typically represented by the long-term cumulated reward. 
In the case of targeted poisoning, 
$\mathcal{L}_A$ is a policy matrix distance, 
measuring the difference between the learned policy ${\theta}_{(0)}^{T}$ and some targeted policy $\theta^{\dagger}$. 
This captures the attacker's goal to manipulate the global model to align with a specifically targeted policy.

\begin{align}
     \arg\min_{\widehat{O}}&\ \mathcal{L}_{A}\Big({\theta}_{(0)}^{T} \Big| \big\{\widehat{\mathcal{O}}_{(n)}^{p,q}\big\}_{1 \leq p \leq T}^{1 \leq q \leq L}\Big)
     \tag{P}\label{opt-P}\ \\
    \text{s.t. } & \forall i \notin \mathcal{M}:\\
    & \theta_{(i)}^{p,0} = \mathcal{A}^{up}(\theta_{(i)}^{p-1,L}, {\theta}_{(0)}^{p-1})
    \label{eq:gen-i-init}\\
    &\theta_{(i)}^{p,q} = \argmax_{\theta} J(\theta, {\theta}_{(i)}^{p,q-1}, \mathcal{O}_{(i)}^{p, q-1})
    \label{eq:gen-i-local}\\
    & \forall j \in \mathcal{M}:\\
    & \theta_{(j)}^{p,0} = \mathcal{A}^{up}(\theta_{(j)}^{p-1,L} ,{\theta}_{(0)}^{p-1})
    \label{eq:gen-n-init}\\
    & {\theta}_{(j)}^{p,q} = \argmax_{\theta} J(\theta, {\theta}_{(j)}^{p,q-1}, \widehat{\mathcal{O}}_{(j)}^{p, q-1})
    \label{eq:gen-n-local}\\
    & D(\widehat{\mathcal{O}}_{(j)}^{p, q} , \mathcal{O}_{(j)}^{p, q})\leq \epsilon
    \label{eq:gen-bgt2}\\
    & {\theta}_{(0)}^{p} = \mathcal{A}^{agg}({\theta}_{(0)}^{p-1}, \{{\theta}_{(i)}^{p,L}\}_{i\in [n]}).
    \label{eq:gen-agg}
\end{align}

\noindent
\textbf{Constraints and knowledge.} The constraints in optimization (\ref{opt-P}) characterize the process of poisoning FRL, and their interpretations are summarized in Table \ref{tab:eq-gen}. The constraint (\ref{eq:gen-i-init}) describes how a clean agent $i$ initializes its policy for local training at the beginning of each round. The agent updates its model from $\theta_{(i)}^{p-1, L}$ to $\theta_{(i)}^{p,0}$ using an update algorithm $\mathcal{A}^{up}(\cdot,\cdot)$. This update algorithm takes the agent's current model $\theta_{(i)}^{p-1, L}$ and the global model broadcast by the server $\theta_{(0)}^{p-1}$ as inputs, and outputs the agent's new local model $\theta_{(i)}^{p,0}$. Constraint (\ref{eq:gen-i-local}) describes the local update of each clean agent at local step $q$ in round $p$. The agent uses its current policy, represented by $\theta_{(i)}^{p,q-1}$, to roll out an observation sequence $O_{(i)}^{p,q-1}$. Then, based on this observation, the agent updates its policy from $\theta_{(i)}^{p,q-1}$ to $\theta_{(i)}^{p,q}$ by maximizing an objective function $J(\cdot)$, defined by the agent's specific RL algorithm. Moreover, constraints (\ref{eq:gen-n-init}) and (\ref{eq:gen-n-local}) characterize the malicious agent's local initialization and update, respectively. Their interpretation is similar to the constraints (\ref{eq:gen-i-init}) and (\ref{eq:gen-i-local}), except that the local update of a malicious agent is based on the manipulated observation $\widehat{O}$. Constraint (\ref{eq:gen-bgt2}) captures the budget constraint for the attacker, where $D(\cdot, \cdot)$ represents the distance between the perturbed (poisoned) observations and the clean observations, which is restricted by a cost budget $\epsilon$. The attack power $\epsilon$ represents the degree of poisoning that is allowed, and it is used to define the limit of the change that can be made to the clean data. Finally, the constraint (\ref{eq:gen-agg}) models the aggregation step of the central server by an aggregation algorithm $\mathcal{A}^{agg}$. The inputs of $\mathcal{A}^{agg}$ are the server's current model $\theta_{(0)}^{p-1}$ and the models sent from agents $\{\theta_{(i)}^{p-1,L}\}$. The output of $\mathcal{A}^{agg}$ is the server's new global model $\theta_{(0)}^{p}$.

\begin{table}[h]
    \caption{Constraints of Problem (\ref{opt-P}).}
    \label{tab:eq-gen}
    \centering
    \begin{small}
    \begin{sc}
    \begin{tabular}{ll}
    \toprule
    Party & Constraints \\
    \midrule
    Agent $(i)$, $i\notin \mathcal{M}$ & Eq.(\ref{eq:gen-i-init}): Local initialization \\
     (Clean agents)& Eq.(\ref{eq:gen-i-local}): Local training \\
     \midrule
    Agent $(j)$, $j\in \mathcal{M}$ &  Eq.(\ref{eq:gen-n-init}): Local initialization\\
    (malicious agents)& Eq.(\ref{eq:gen-bgt2}): Attack Budget\\
    &Eq.(\ref{eq:gen-n-local}): Local training\\
    \midrule
    Coordinator & Eq.(\ref{eq:gen-agg}): aggregation \\
    \bottomrule
    \end{tabular}
    \end{sc}
    \end{small}
\end{table}

Table \ref{tab:knowledge} shows the knowledge of the three parties involved in FL, namely, the coordinator, the attacker, and the agents. The table clarifies the knowledge of each party to guarantee the data privacy expected from FL. The coordinator only has knowledge of the submitted models and the global policy, while the agents only have knowledge of their local data, policy, and the broadcast global model. The attacker only has knowledge of its own observations, manipulations, model, and the broadcast global policy. This ensures that the agents' private data is kept confidential.

\begin{table}[h]
    \caption{Knowledge of the parties in a poisoned FRL.}
    \label{tab:knowledge}
    \centering
    \resizebox{0.45\textwidth}{!}{ 
    \begin{small}
    \begin{sc}
    \begin{tabular}{lccc}
    \toprule
    & Coordinator & Agent $(i), i\notin \mathcal{M}$  & Agent $(j), j\in \mathcal{M}$  \\
    \midrule
     $\mathcal{L}_{\mathcal{A}}(\cdot)$ & & & $\surd$\\
     $J(\cdot)$ &   & $\surd$ & $\surd$\\
     $\mathcal{A}^{up}(\cdot)$ &&$\surd$&$\surd$\\
     $\mathcal{A}^{agg}(\cdot)$ &$\surd$ &&\\
    \midrule
     $\theta_{(0)}^p$& $\surd$ & $\surd$ & $\surd$\\
     $\theta_{(i)}^{p,q}, i\notin \mathcal{M}$ && $\surd$ &\\
     $\theta_{(j)}^{p,q}, j\in \mathcal{M}$ &&  & $\surd$\\
    \midrule
     $\epsilon$&& & $\surd$\\
    \midrule
     $\mathcal{O}_{(i)}^{p,q}, i\notin \mathcal{M}$ &&$\surd$ &\\
     $\mathcal{O}_{(j)}^{p,q}, j \in \mathcal{M}$ &&&$\surd$\\
     $\widehat{\mathcal{O}}_{(j)}^{p,q},j \in \mathcal{M}$ &&& $\surd$\\
    \bottomrule
    \end{tabular}
    \end{sc}
    \end{small}
    }
\end{table} 

We refer to Appendix A for a specification of this framework to the case of Proximal Policy Optimization ~\citep{schulman2017proximal}, which shows how our framework characterizes poisoning-FRL for actor-critic algorithms.
\section{Method}
\label{sec:method}
In this section, we propose practical local environment poisoning methods to FRL. We consider two scenarios when the local RL training/update in the FRL system uses actor-critic methods (Algorithm \ref{alg:poison}) and policy gradient methods (Algorithm \ref{alg:poison_2}). Here, we only describe untargeted reward poisoning methods as other variants can be easily derived from our general formulation in Section \ref{sec:PF}. Finally, in the Appendix, we provide a defense mechanism against the proposed poisoning protocols, which is based on assigning credits to each agent's policy according to its performance. To describe our poisoning methods, as before, we use $\ \widehat{\cdot}\ $ to denote the malicious agents' parameters communicated with the server.

\begin{algorithm}[H]
   \caption{Reward Poisoning with Actor-Critic}
   \label{alg:poison}
\begin{algorithmic}[1]
   \STATE {\bfseries Input}: max federated rounds $T$, max local episodes $L$, number of agents $n$, index set of malicious agents $\mathcal{M}$, poisoning budget $\eps$, aggregation algorithm $\mathcal{A}^{agg}$, actor-critic objective function $J$.\\
   \STATE {\bfseries Output}: server's actor and critic models $\theta_{(0)}^{T}$ and $\omega_{(0)}^{T}$. 
   \STATE Initialize the server's actor model $\theta_{(0)}^{0}$.
   \STATE Initialize the server's critic model $\omega_{(0)}^{0}$.
   \FOR{$p=1$ {\bfseries to} $T$}
   \FOR{$i=1$ {\bfseries to} $n$}
   \IF{$i\notin \mathcal{M}$}
   \STATE Initialize local critic $\omega_{(i)}^{p,0} \leftarrow \omega_{(0)}^{p-1}$
   \STATE Initialize local actor $\theta_{(i)}^{p,0} \leftarrow \theta_{(0)}^{p-1}$
   \ELSE
   \STATE (Attacker) Initialize local private critic $\omega_{(i)}^{p,0} \leftarrow \omega_{(i)}^{p-1,L}$
   \STATE Initialize local public critic $\widehat{\omega}_{(i)}^{p,0}
   \leftarrow \omega_{(0)}^{p-1}$\label{alg:public-critic-init}
   \STATE Initialize local actor $\widehat{\theta}_{(i)}^{p,0} \leftarrow \theta_{(0)}^{p-1}$\label{alg:private-critic-init}
   \ENDIF
   \FOR{$q=1$ {\bfseries to} $L$}
   \IF{$i \notin \mathcal{M}$}
   \STATE Interact with environment and obtain ${O}_{(i)}^{p,q}$
   \STATE Compute $J_{(i)}^{p,q}$ with ${O}_{(i)}^{p,q}$ and $\omega_{(i)}^{p,q-1}$ 
   \STATE Update ${\theta}_{(i)}^{p,q}$ with $J_{(i)}^{p,q}$
   \STATE Update ${\omega}_{(i)}^{p,q}$ with ${O}_{(i)}^{p,q}$
   
   \ELSE
   \STATE (Attacker) Interact with environment and obtain ${O}_{(i)}^{p,q}$
   \STATE Compute $J_{(i)}^{p,q}$ with ${O}_{(i)}^{p,q}$ and $\omega_{(i)}^{p,q-1}$
   \STATE Poison Reward as $\{\widehat{r}_t\}_{(i)}^{p,q}$ by Eq. (\ref{eq:poison_r})
   
   \STATE Obtain $\widehat{J}_{(i)}^{p,q}$ with $\widehat{O}_{(i)}^{p,q}$ and $\widehat{\omega}_{(i)}^{p,q-1}$ \label{lin:att-obs}
   
   \STATE Update $\widehat{\theta}_{(i)}^{p,q}$ with $\widehat{J}_{(i)}^{p,q}$
   \STATE Update $\widehat{\omega}_{(i)}^{p,q}$ with $\widehat{O}_{(i)}^{p,q}$
   \STATE Update ${\omega}_{(i)}^{p,q}$ with ${O}_{(i)}^{p,q}$
   
   \ENDIF
   \ENDFOR 
   \ENDFOR 
   \STATE $\theta_{(0)}^p = \mathcal{A}^{agg}\big(\{\widehat{\theta}_{(i)}^{p,L}\}_{i\in \mathcal{M}}\bigcup \{\theta_{(i)}^{p,L}\}_{i\notin \mathcal{M}}\big)$
   \STATE $\omega_{(0)}^p = \mathcal{A}^{agg}\big(\{\widehat{\omega}_{(i)}^{p,L}\}_{i\in \mathcal{M}}\bigcup \{\omega_{(i)}^{p,L}\}_{i\notin \mathcal{M}}\big)$ \label{alg:public-critic-upd}
   \ENDFOR 
\end{algorithmic}
\end{algorithm}

\subsection{Reward Poisoning with Local Actor-Critic Algorithms}
\label{subsec:method-poison}
In actor-critic algorithms~\citep{peters2008natural}, besides the policy $\theta$, each agent also uses a critic model $\phi_{\omega}(\cdot)$ parameterized by $\omega$ to estimate the value function $V_{\theta}(s)$ as $\overline{V}(s) = \phi_{\omega}(s)$. By abuse of notation, we again refer to the critic model $\phi_{\omega}$ by its parametrization $\omega$. This also gives an estimation of the $Q$-function as $\overline{Q}(s,a) = r(s,a) + \gamma \cdot \overline{V}(s')$, where $r(s,a)$ is the observed reward, and $s'$ is the next observed state. The critic model updates itself by minimizing the temporal-difference error~\citep{tesauro1995temporal}. The policy parameter $\theta$ is updated by $\theta^{t+1} =  \arg\max_{\theta}J(\theta,\theta^t,\mathcal{O}^{t})$, where $J(\cdot)$ is some objective function specified by the actor-critic algorithm and $\mathcal{O}^{t}$ is the observed trajectory. In the following, we describe the poisoning mechanism in detail, which is also summarized in Algorithm (\ref{alg:poison}).

\smallskip
\noindent
\textbf{Public and private adversarial critics.} 
In order to manipulate the global model by training and submitting a poisoned critic model, the attacker needs an unpoisoned critic to provide a true estimation of the value function (V-value) to make the best poison decision. To address this need, the attacker trains a pair of public and private critics, $\widehat{\omega}_{(j)}^{p,q}$ and ${\omega}_{(j)}^{p,q}$. The public critic is trained with poisoned rewards and used to communicate with the server, while the private critic is trained with ground-truth observations and kept to itself. The private critic is used to obtain a true estimation of the V-value, based on which the attacker can decide how to manipulate the rewards. Due to the different goals and training methods of the pair of critics, the initialization for the critics is also different. At the beginning of a new round of local training, the agent initializes its public critic with the broadcast global model (i.e., $\widehat{\omega}_{(j)}^{p,0} = {\omega}_{(0)}^{p-1}$), while inherits its private critic from the end of last round of training (i.e., ${\omega}_{(j)}^{p,0} = {\omega}_{(j)}^{p-1,L}$). 

\smallskip
\noindent
\textbf{Reward poisoning.}
In contrast to targeted poisoning, where the attacker forces the coordinator to learn a specific policy, in untargeted poisoning, the attacker aims to minimize the objective function $J$ of the policy model (the actor) by poisoning its local rewards. The actor is updated according to the poisoned objective $\widehat{J}$ obtained by poisoned rewards $\widehat{r}$ and public critic $\widehat{\omega}$. However, when deciding how to poison the rewards, the attacker should use the unpoisoned objective $J$ given by the ground-truth rewards $r$ and private critic $\omega$, from which the attacker obtains true estimation in order to make a right poisoning decision. The attacker poisons the reward $r$ received by the corrupted agent by minimizing the original objective $J$ with respect to $r$.

Concretely, for each round $p\leq T$ and local step $q\leq L$, the corrupted agent $j\in \mathcal{M}$ interacts with the environment and obtains the ground-truth observation ${O}_{(j)}^{p,q}$. The attacker computes the unpoisoned objective $J_{(j)}^{p,q}$ using $\{{r}_t\}_{(j)}^{p,q}$ and $\omega_{(j)}^{p,q}$. Then, the attacker poisons the reward by
\begin{align}
    \{\widehat{r}_t\}_{(j)}^{p,q}\leftarrow \{r_t\}_{(j)}^{p,q} -  \eps \cdot \frac{\nabla_{r}J_{(j)}^{p,q}}{\|\nabla_{r}J_{(j)}^{p,q}\|}
    \label{eq:poison_r},\
\end{align}
which guarantees that the manipulation power is within the attack budget $\epsilon$.

\subsection{Reward Poisoning with Local Policy Gradient (PG) Algorithms}
\label{subsec:method-poison2}
In Policy Gradient (PG) algorithms~\citep{silver2014deterministic}, agents do not require a critic model. Therefore, we have adapted the poisoning method described in Section \ref{subsec:method-poison} that uses Actor-Critic for agents' local RL algorithm. The attacker's goal is still to minimize the objective function $J$ of the policy model. Since there is no critic, the agent directly calculates $J$ from the observed rewards ${r}$ and uses this information to decide how to poison the rewards to ${\widehat{r}}$ by Eq. (\ref{eq:poison_r}). The policy is then updated with the poisoned objective $\widehat{J}$ calculated by the poisoned rewards. The overall procedure is outlined in Algorithm (\ref{alg:poison_2}).

\begin{algorithm}[H]
   \caption{Reward Poisoning with Policy Gradient}
   \label{alg:poison_2}
\begin{algorithmic}[1]
   \STATE {\bfseries Input}: max federated rounds $T$, max local episodes $L$, number of agents $n$, index set of malicious agents $\mathcal{M}$, poisoning budget $\eps$, aggregation algorithm $\mathcal{A}^{agg}$, policy gradient objective function $J$.
   \STATE {\bfseries Output}: server's policy $\theta_{(0)}^{T}$. 
   \STATE Initialize the server's policy $\theta_{(0)}^{0}$.
   \FOR{$p=1$ {\bfseries to} $T$}
   \FOR{$i=1$ {\bfseries to} $n$}
   \STATE Initialize local policy $\theta_{(i)}^{p,0} \leftarrow \theta_{(0)}^{p-1}$

   \FOR{$q=1$ {\bfseries to} $L$}
   \IF{$i \notin \mathcal{M}$}
   \STATE Interact with environment and obtain ${O}_{(i)}^{p,q}$
   \STATE Compute $J_{(i)}^{p,q}$ with ${O}_{(i)}^{p,q}$
   \STATE Update ${\theta}_{(i)}^{p,q}$ with $J_{(i)}^{p,q}$
   
   \ELSE
   \STATE (Attacker) Interact with environment and obtain ${O}_{(i)}^{p,q}$
   \STATE Compute $J_{(i)}^{p,q}$ with ${O}_{(i)}^{p,q}$
   \STATE Poison Reward as $\{\widehat{r}_t\}_{(i)}^{p,q}$ by Eq. (\ref{eq:poison_r})

   \STATE Obtain $\widehat{J}_{(i)}^{p,q}$ with $\widehat{O}_{(i)}^{p,q}$
   
   \STATE Update $\widehat{\theta}_{(i)}^{p,q}$ with $\widehat{J}_{(i)}^{p,q}$
   
   \ENDIF
   \ENDFOR 
   \ENDFOR 
   \STATE $\theta_{(0)}^p = \mathcal{A}^{agg}\big(\{\widehat{\theta}_{(i)}^{p,L}\}_{i\in \mathcal{M}} \bigcup \{\theta_{(i)}^{p,L}\}_{i\notin \mathcal{M}}\big)$
   \ENDFOR 
\end{algorithmic}
\end{algorithm}

\section{A Theoretical Performance Guarantee}
In this section, we prove that our poisoning method can strictly decrease the global objective under certain assumptions. We begin by considering the following assumption. 

\begin{figure*}[ht]
    \centering
    \includegraphics[height=1.29in]{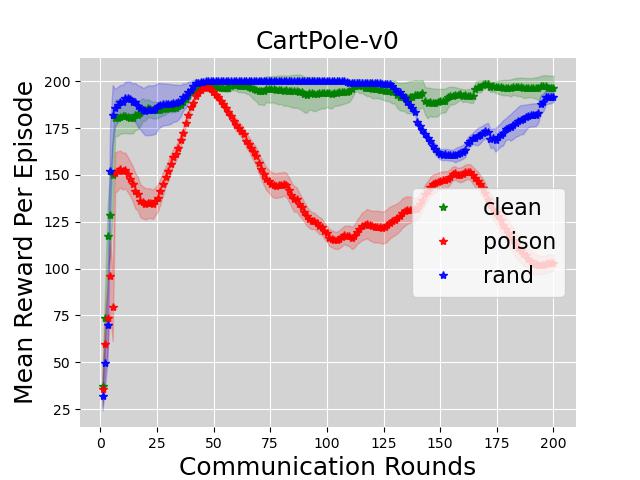}
    \includegraphics[height=1.29in]{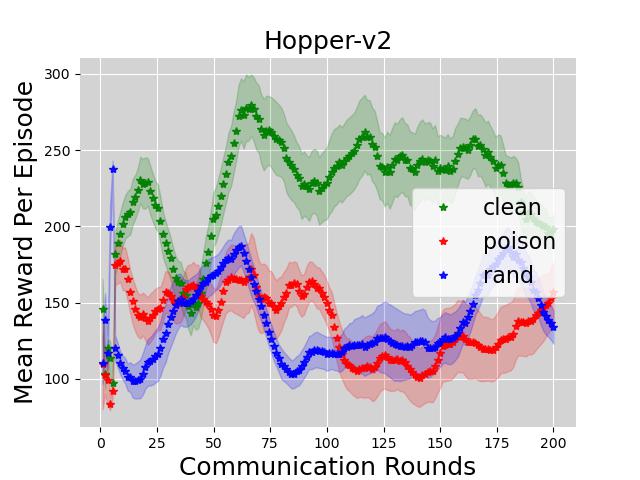}
    \includegraphics[height=1.29in]{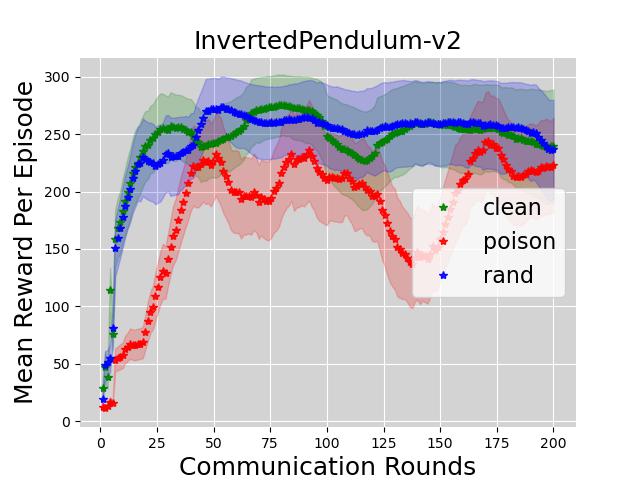}
    \includegraphics[height=1.29in]{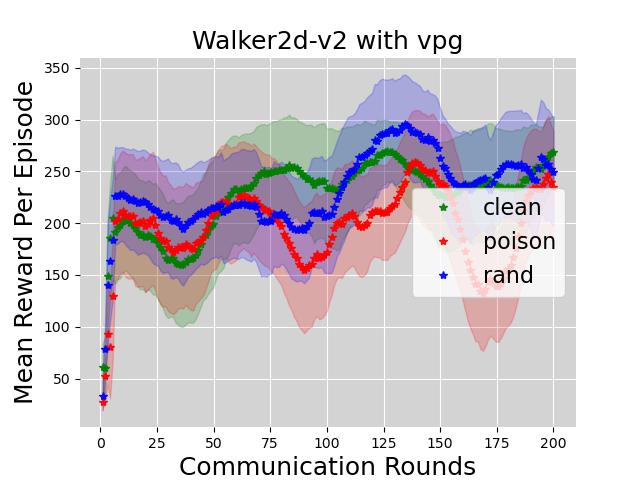}
    \caption{Contrastive performance for VPG FRL system. We plot performance of local environment poisoning against two baselines of clean training and random attack. We only report the largest size of the system that the attacker can successfully poison, which are 4 for the left two plots and 3 for the right two.}
    \label{fig:vpg-r1}
\end{figure*}

\begin{assumption}[Single-step local training, single-agent poisoning]
\label{asp:single}
    We assume that for each round, all local agents only apply single-step local training, i.e., $L=1$. Moreover, we assume that only agent $(n)$ is poisoned, i.e., $\gM = \{n\}$.
\end{assumption}

\noindent
\textbf{Clean local reward sequence function.} Let us denote the dimensions of $\theta$ and $r$ by $d_\theta$ and $d_r$, respectively. Given a policy $\theta$, define $r_{(i)}(\cdot): \sR^{d_\theta} \rightarrow \sR^{d_r}$ as the function of the reward sequence generated by the agent interacting with the local environment $(i)$ under policy $\theta$. Under Assumption \ref{asp:single}, denote the reward sequence observed by agent $i$ during round $p$ by $r_{(i)}^p$. Then, we have $r_{(i)}^p = r_{(i)}(\theta_{(0)}^{p-1})$.

\noindent
\textbf{Clean local objective function.} With slight abuse of notations, let $J(\theta;r)$ denote the local RL objective function given the policy $\theta$ and the reward sequence $r$, where we have $J:\sR^{d_\theta\times d_r}\to \sR$. According to the mechanism of FRL, we further define the objective of local agent $(i)$ given an initialized policy $\theta$ as $J_{(i)}(\theta):= J(\theta; r_{(i)}(\theta))$. If we denote the objective of local agent $(i)$ at the \textit{start} of round $p$ by $J_{(i)}^p$, we have $ J_{(i)}^p= J(\theta_{(0)}^{p-1}; r_{(i)}(\theta_{(0)}^{p-1})) = J(\theta_{(0)}^{p-1}; r_{(i)}^p)$.

\noindent
\textbf{Poisoned local reward sequence and local objective}. Under Assumption \ref{asp:single}, agent $(n)$ is poisoned. According to our algorithm, we have $\hat{r}_{(n)}^p = {r}_{(n)}^p - \eps \cdot \vec{e}(\nabla_{r_{(n)}^p} J_{(n)}^p)$. Therefore, the poisoned objective is given by $\hat{J}_{(n)}^p = J(\theta_{(0)}^{p-1}; \hat{r}_{(n)}^p)$.

\noindent
\textbf{Local policy update function.} Under Assumption \ref{asp:single}, each agent performs only a one-step local update. Since only agent $(n)$ is poisoned, we have 
\begin{align}\nonumber
\hat{\theta}_{(n)}^p = \theta_{(0)}^{p-1} + \lambda_{\theta} \cdot \nabla_{\theta_{(0)}^{p-1}} \hat{J}_{(n)}^p.  
\end{align}

\begin{assumption}[FedAVG]
\label{asp:fedavg}
    Suppose the server updates global policy by the conventional FedAVG settings, i.e., $\theta_{(0)}^{p} = \frac{1}{n}\sum_{i\in [n]} \theta_{(i)}^{p}$ for clean training. In particular, when agent $(n)$ is poisoned, we have $\hat{\theta}_{(0)}^{p} = \frac{1}{n}(\sum_{i\in [n-1]} \theta_{(i)}^{p} + \hat{\theta}_{(n)}^p )$.
\end{assumption}

\noindent
\textbf{Global objective function.} Under the conventional setting of FL, given model $\theta$, the global objective is defined as $J_{(0)}(\theta) := \frac{1}{n}\sum_{i = 1}^n J_{(i)}(\theta)$. In FRL, at the \textit{end} of round $p$, the clean federated objective is given by  $J_{(0)}^p =J_{(0)}(\theta_{(0)}^p)$. When agent $(n)$ is poisoned, we have $\hat{J}_{(0)}^p =J_{(0)}(\hat{\theta}_{(0)}^p)$.

\begin{assumption}[Objective smoothness]
\label{asp:smooth}
We assume that $J_{(0)}^p$ is differentiable with respect to $r$ and $\theta$ almost everywhere, and $J_{(0)}^p$ is $L_r-$smooth with respect to $r_{(n)}^p$.
\end{assumption}

We are now ready to state our main theoretical result, whose proof is deferred to Appendix A.

\begin{theorem}
\label{thrm:UB}
Let Assumptions \ref{asp:single}, \ref{asp:fedavg}, and \ref{asp:smooth} hold. Suppose that all agents are updated cleanly at the first $p-1$ rounds, and at round $p$, agent $(n)$ is poisoned. Define a scalar $\eps_+ := \frac{2\lambda_\theta B}{n L_r }$, where $B$ is a scalar defined as
\begin{align*}
    B= [\nabla_{\theta'}  J_{(0)}(\theta')]^\top
    \cdot 
    [\nabla_{r,\theta} J(\theta, r)] 
    \cdot \vec{e}(\nabla_{r} J(\theta, r))\Big|_{r= r_{(n)}^p}^{\substack{\theta = \theta_{(0)}^{p-1}\\ \theta'=\theta_{(0)}^p}}.
\end{align*}
Then, for $B>0$ and $\eps < \eps_+$, we have
\begin{align}\nonumber
    \widehat{J}_{(0)}^p < 
    J_{(0)}^p  -  \alpha,
\end{align}
where $ \alpha \in[0, \frac{\eps_+^2}{8}]$.
\end{theorem}

According to our poisoning framework, with a small poison budget $\epsilon$, we can guarantee that the poisoned global objective $\hat{J}_{(0)}^p$ is strictly smaller than the clean global objective ${J}_{(0)}^p$. The upper bound of the objective gap $\alpha$ is determined by the upper bound of $\epsilon$. Therefore, a higher attack budget $\eps$ can indicate a greater decrease in the global objective. Similarly, a larger local learning rate $\lambda_\theta$ or fewer number of agents can increase $\eps_+$, indicating a stronger objective decrease $\alpha$. In Appendix A, we shall discuss a practical scenario under which the condition $B>0$ always holds.

\section{Numerical Experiments}
\label{sec:Eval}

In this section, we conduct a series of experiments to evaluate the effectiveness of our poisoning method on the FRL system. Our results show that the proposed poisoning method can effectively reduce the mean episode reward of the server in the FRL system. Additionally, our poisoning protocol does not require knowledge of the MDP of the environment and is consistent with the multiple local steps setting of FRL.

\begin{figure*}[htp]
    \centering
    \includegraphics[height=1.29in]{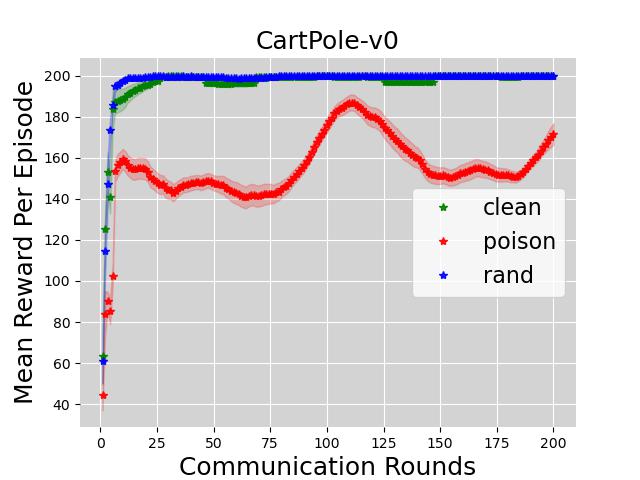}
    \includegraphics[height=1.29in]{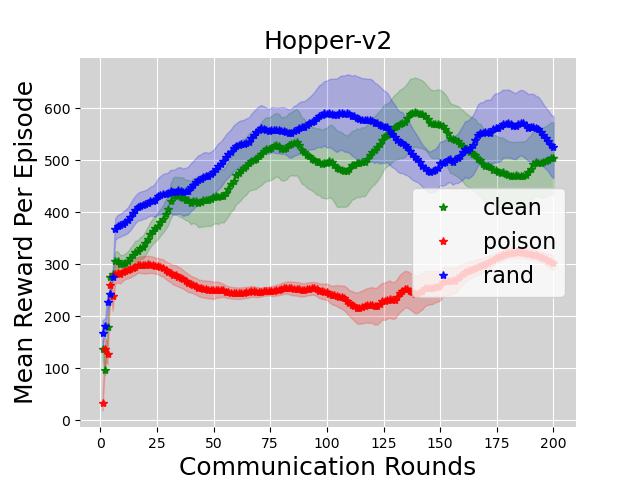}
    \includegraphics[height=1.29in]{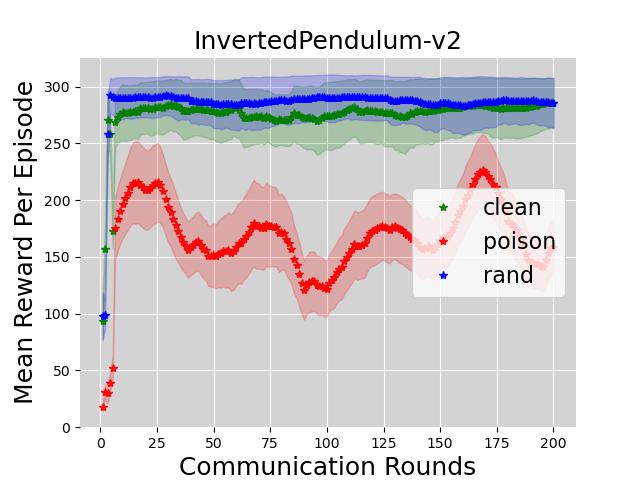}
    \includegraphics[height=1.29in]{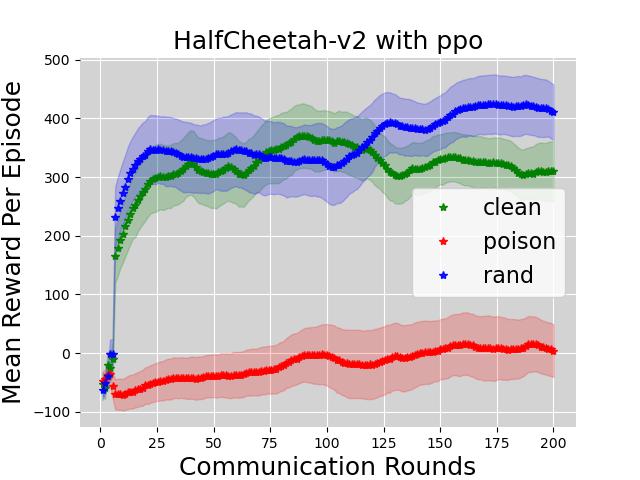}
    \caption{Rewards given by a poisoned PPO system (attack budget $\eps = 1$) with a single attacker under the proposed method are significantly lower than the clean system of the same agent size and random attacks. The system size is three agents for InvertedPendulum and four agents for the others.}
    \label{fig:ppo-r1}
\end{figure*}

\subsection{Experimental Settings}

\noindent
\textbf{Local RL settings.} We train our individual-agent model with two reinforcement learning algorithms: Vanilla Policy Gradient (VPG)~\cite{sutton1999policy} and Proximal Policy Optimization (PPO)~\cite{schulman2017proximal}.
We evaluate the system on various OpenAI Gym environments with increasing complexity, including CartPole, Inverted Pendulum, Hopper, Lunar Lander, and Half Cheetah. For both VPG and PPO settings, we let the malicious agent attacks the reward with a budget of $\eps = 1$; all agents run $50$ local steps in each communication round. 

\smallskip
\noindent
\textbf{Federated settings.} We follow a conventional federated framework. In particular, the central server aggregates the models submitted by local agents by taking the average at the end of each communication round (Eq.(\ref{eq:A-agg-spe})) and then broadcasts the new global model, which is used by the local agents as initialization (Eq.(\ref{eq:A-up-spe})) for the next round of local training. There are $200$ total communication rounds.\footnote{We refer to Appendix A for additional numerical results related to other settings of the FRL systems.}

\smallskip
\noindent 
\textbf{Baselines.} To the best of our knowledge, no existing FRL-poisoning method matches our setting of multiple local steps without access to the MDP environment. Therefore, to evaluate our method, we compare it to a baseline of clean training, where all agents follow the expected rules of local updates and federated communications, i.e., $|\mathcal{M}| = 0$. We also test a reward-based randomized attack, in which (\ref{eq:poison_r}) is replaced by 
\begin{align}
    \{\widehat{r}_t\}_{(j)}^{p,q}\leftarrow \{r_t\}_{(j)}^{p,q} -  \eps \cdot x
    \text{ for } x \sim U(0, 1),
    \label{eq:poison_rand}
\end{align}  
where $U(a,b)$ denotes the uniform distribution from $a$ to $b$. 

\smallskip
\noindent
\textbf{Evaluation matrix.} For \textit{untargeted poisoning}, we evaluate the performance of these methods by measuring the mean-episode reward of the central model, which is calculated based on 100 test episodes at the end of each federated round. For \textit{targeted poisoning}, we measure the similarity between learned policy and targeted policy: for discrete action space, we calculate the proportion of target actions among all actions; For continuous action space, we collect $1- scaled\ distance$. Under both measurements, a higher value indicates a closer learned policy to the target policy. For all experiments, we average the results over multiple random seeds.

\subsection{Performance Evaluation}
Here, we present experimental results that demonstrate the effectiveness of our proposed poisoning method compared to a clean setup, compare it to other poisoning methods, and assess the impact of the system size on the success of the attack. We show that with a small proportion of malicious agents, our method is able to poison most FRL systems successfully. Additionally, we observe its consistent outperformance compared to other poisoning methods across algorithms and environments. Lastly, we demonstrate that the proportion of malicious agents remains a critical factor in the success of the attack, regardless of the overall size of the system.

For the VPG algorithm (Fig. (\ref{fig:vpg-r1})), we see that across several environments, the attacker successfully poisons a system of 3-4 agents using our method. The effectiveness of this poisoning depends on the environment. For instance, we can observe a very clear difference between clean and poisoned runs for CartPole, but not as much for InvertedPendulum. Nevertheless, there is less consistency in the rewards for InvertedPendulum, and as the rounds progress, the clean run would continue to accumulate greater mean rewards. We also observe that our method performs at least as well and, in most cases, much better than a random attack in poisoning the system. This proves the capability of our VPG attack method in poisoning federated systems.

Attacking a PPO system (Fig. (\ref{fig:ppo-r1})) lends itself to a much different dynamic because the separate value parameters allow all agents to learn more effectively, and thus we expect both clean and malicious agents to be more successful in their opposite goals. Ultimately, we observe that malicious agents using our method are much more successful in their attacks across environments compared to the VPG results. We attribute this to the more sophisticated attack model, which reduces noise in the gradient updates and enables the reward poisoning to be more effective. In contrast, we see that a random attack performs very poorly, as the strength of the clean agents overwhelms the weak attack and results in high rewards over time. 

Finally, we conduct extensive experiments, investigating other attack aspects on the performance of the FRL, such as larger attack budget, attacker proportion generalization, targeted attack, and single-critic attack. In particular, we provide a defense mechanism to robustify the FRL system under certain attack environments. Due to space limitations, we defer all those results to the supplementary document. 

\section{Conclusions}
In this work, we propose a novel method for poisoning FRL under both general policy-based and actor-critic algorithms, which can provably decrease the system's global objective. Our method is evaluated through extensive experiments on various OpenGYM environments using popular RL models, and the results demonstrate that our method is effective in poisoning FRL systems. Our work highlights the potential risks of FRL and inspires future research for designing more robust algorithms to protect FRL against poisoning attacks.

\bibliography{ref}
\newpage

.
\newpage 
\section*{Appendix A: Proof of Theorem \ref{thrm:UB}}
Lemma: quadratic inequality
\medskip
\begin{theorem}[Theorem \ref{thrm:UB}-restated]
\label{thrm:UB-proof}
Let Assumptions \ref{asp:single}, \ref{asp:fedavg}, and \ref{asp:smooth} hold. Suppose that all agents are updated cleanly at the first $p-1$ rounds, and at round $p$, agent $(n)$ is poisoned. Define a scalar $\eps_+ := \frac{2\lambda_\theta B}{n L_r }$, where $B$ is a scalar defined as
\begin{align*}
    B= [\nabla_{\theta'}  J_{(0)}(\theta')]^\top
    \cdot 
    [\nabla_{r,\theta} J(\theta, r)] 
    \cdot \vec{e}(\nabla_{r} J(\theta, r))\Big|_{r= r_{(n)}^p}^{\substack{\theta = \theta_{(0)}^{p-1}\\ \theta'=\theta_{(0)}^p}}.
\end{align*}
Then, for $B>0$ and $\eps < \eps_+$, we have

\begin{align}
    \widehat{J}_{(0)}^p < 
    J_{(0)}^p  -  \alpha, \label{eq:UB-ntg}
\end{align}
where $ \alpha \in[0, \frac{\eps_+^2}{8}]$.
\end{theorem}

\begin{proof}
Since $J_{(0)}^p$ is $L_r$-smooth with respect to $r_{(n)}^p$, we have
Moreover, we can write
and
Substituting Eq. (\ref{eq:and}) into Eq. (\ref{eq:since}), we obtain
On the other hand, we have
If we combine Eq. (\ref{eq:since2}) with Eq. (\ref{eq:get}), we get
Substituting Eq. (\ref{eq:combine}) into Eq. (\ref{eq:lip}) and using the fact that $\frac{L_r}{2} \|\widehat{r}_{(n)}^p - r_{(n)}^p\|^2 
= \frac{ \epsilon^2\cdot L_r}{2}$, we get
As the right-hand side of Eq. (\ref{eq:lip'}) is a quadratic function with respect to $\epsilon$, we get that when $B > 0$ and $0< \epsilon < \frac{2 \lambda_\theta B}{K L_r}$, it holds that $\hat{J}_{(0)}^p < J_{(0)}^p$, indicating a strict decrease in the objective of FRL being poisoned compared with clean training. In particular, the smallest bound in Eq. (\ref{eq:lip'}) is acheived when $\epsilon = \frac{\lambda_\theta B}{n L_r}$, implying $\widehat{J}_{(0)}^p \leq 
    J_{(0)}^p  -  \frac{\lambda_\theta^2 B^2}{2L_r^2 K^2} = J_{(0)}^p  -  \frac{\eps_+^2}{8}$.\end{proof}

\begin{figure*}[t]
    \centering
    \includegraphics[height=1.29in]{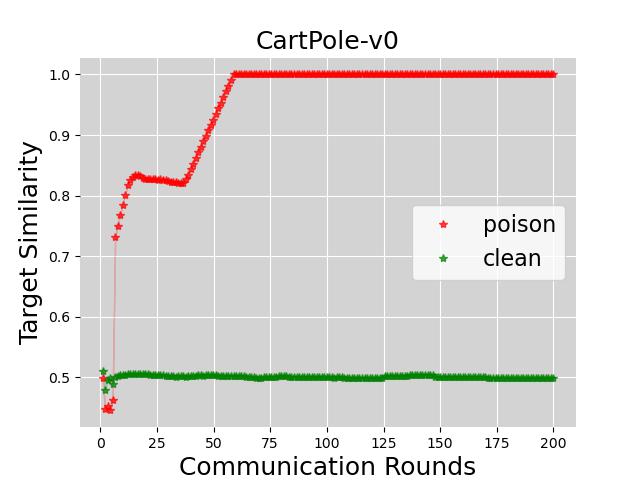}
    \includegraphics[height=1.29in]{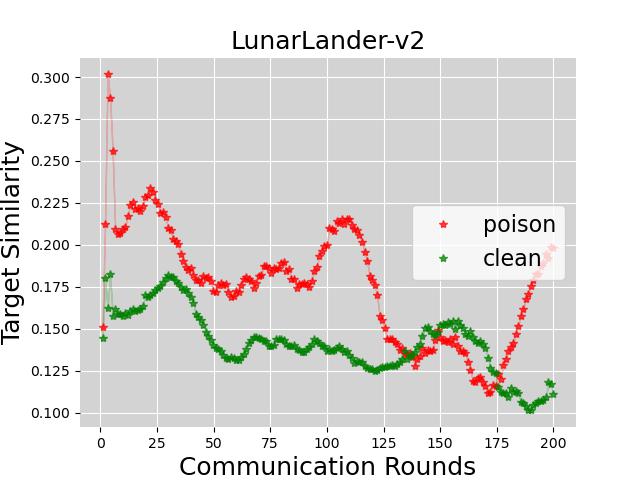}
    \includegraphics[height=1.29in]{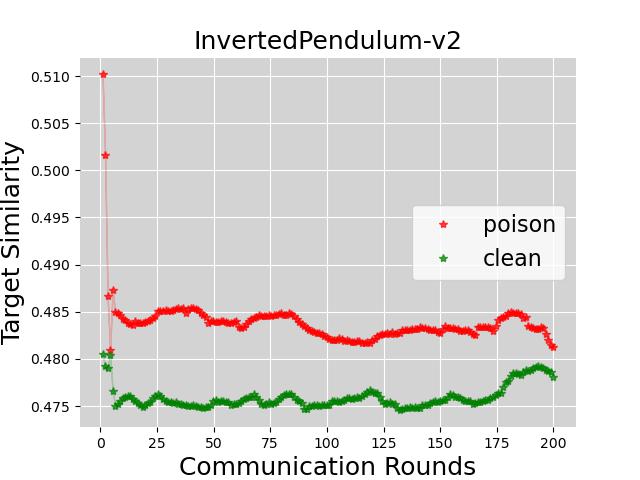}
    \includegraphics[height=1.29in]{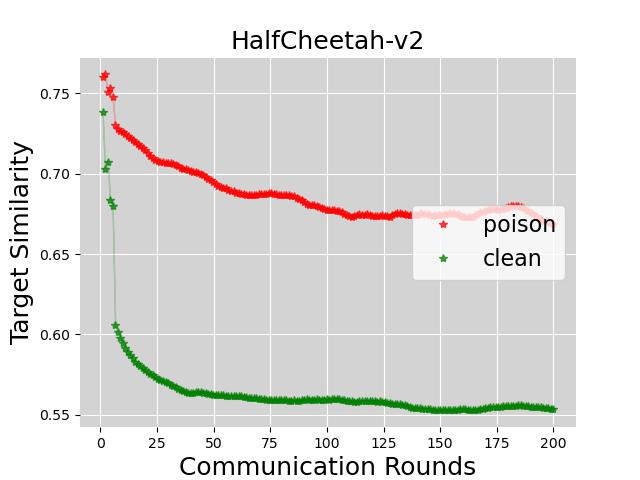}
    \caption{Target attack against VPG FRL. We choose environments of different action space: CartPole (two discrete actions), LunarLander (four discrete actions), InvertedPendulum (one continuous dimension), HalfCheetah (six continuous dimensions).}
    \label{fig:target}
\end{figure*}

\noindent
\textbf{A case where $B>0$.}
Suppose $J(\theta; r)= {\gamma} \cdot r$, where $r$ is the reward sequence and $\gamma$ is the discount factor vector. This objective corresponds with the typical accumulated discounted reward setting. Suppose that $r_{(n)}(\theta)$ is a derivable function. In this case, it is easy to see that
Moreover, direct calculations reveal that
Therefore, 
Suppose that agents interact with the same environment, that is $r_{(i)}(\theta) = r(\theta), \forall i\in[n]$. Then, if $r'''(\theta) = 0$, we get
and also,
Therefore, in this case we always have $B>0$, as long as $r''(\theta_{(0)}^{p-1}) \in (-\infty, -(\lambda \gamma)^{-1})\bigcup(0,+\infty)$.

\section*{Appendix B: Proximal Policy Optimization (PPO)-Specific Framework}
\label{sec:PF2}
In this appendix, we focus on a specific local RL algorithm, Proximal Policy Optimization (PPO)~\cite{schulman2017proximal}, for the individual agents in FRL and propose a corresponding framework for poisoning. By specifying the local RL algorithm as PPO, we are able to tailor the problem formulation in Problem (\ref{opt-P}) and accordingly propose a targeted solution in Section \ref{sec:method} to poisoning PPO-specific FRL. In Section \ref{subsec:PPOpre}, we introduce PPO preliminaries to specify the general variable in Problem (\ref{opt-P}). Then we discuss the PPO-specific problem formulation for poisoning FRL in Section (\ref{subsec:PPOpoison}). This will allow us to take into account the specific characteristics of the PPO algorithm when defining the problem.

\subsection*{PPO Preliminaries} 
\label{subsec:PPOpre}
PPO is a popular Actor-Critic algorithm that uses a clipped surrogate objective. 
For agent $i$ at federated round $p$ and local episode $q$, denote the pair of state and action at its $t$-th step of rollout as $O_{(i)}^{p,q,t} = (s_{(i)}^{p,q,t}, a_{(i)}^{p,q,t})$ and denote the $V$-function and $Q$-function defined by Bellman Equation~\cite{baird1995residual} as $V(\cdot)$ and $Q(\cdot,\cdot)$, respectively. Then the advantage function is $A_{(i)}^{p,q,t}:= Q(s_{(i)}^{p,q,t},a_{(i)}^{p,q,t}) - V(s_{(i)}^{p,q,t})$.

\smallskip
\noindent 
\textbf{PPO's actor.} Denote the actor model as $\pi_{\theta}(\cdot|s,\theta)$, where $\theta$ is the model weight and $s$ is some given state.
To specify the general problem (\ref{opt-P}) to the PPO case, the clean agent's objective $J(\cdot)$ (Eq.(\ref{eq:gen-i-local})) should be the PPO surrogate objective
where$
    {\gamma}_{(i)}^{p,q,t}:= {\pi({a}_{(i)}^{p,q,t}|{s}_{(i)}^{p,q,t}, \theta_{(i)}^{p,q})}
    \big/{\pi({a}_{(i)}^{p,q,t}|{s}_{(i)}^{p,q,t}, \theta_{(i)}^{p,q-1})} $, and $\overline{A}_{(i)}^{p,q,t}$ is estimated based on both PPO's critic and the observation that the actor samples. Here,
    $c_{(i)}^{p,q,t} := \text{clip}(\gamma_{(i)}^{p,q,t}, 1-\eta, 1+\eta)
$, where $\mbox{clip}(\cdot)$ is a clipping function parameterized by $\eta$.

\smallskip
\noindent
\textbf{PPO's critic.} Let us use $\ \overline{\cdot}\ $ to denote estimation. Denote PPO's critic model as $\phi(\cdot\big|\omega_{(i)}^{p,q})$, where $\omega$ is the model's weights. As with all typical actor-critic algorithms, the critic is a Value neural network to help estimate the V-value of the actor so as to further calculate the actor's objective. In PPO, the actor's objective is a clipped advantage (Eq. \ref{eq:ppo-obj-act}), where the advantage is estimated by the critic and observation, which can be written in the form of $\overline{A}_{(i)}^{p,q,t} = A(\omega_{(i)}^{p,q}, O_{(i)}^{p,q,t})$. The critic model updates itself by minimizing the temporal-difference error~\citep{tesauro1995temporal} between the estimated and observed $V$-value. We denote the critic's objective by $\delta_{(i)}^{p,q}$.

\begin{figure}[H]
    \centering
    \includegraphics[height=1.23in]{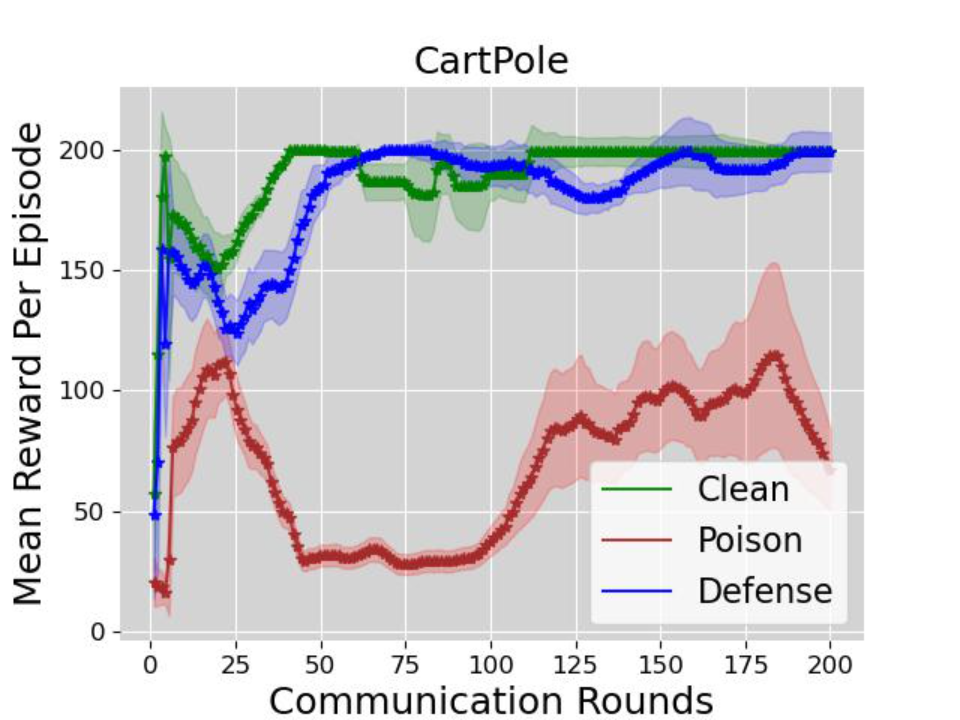}
    \includegraphics[height=1.23in]{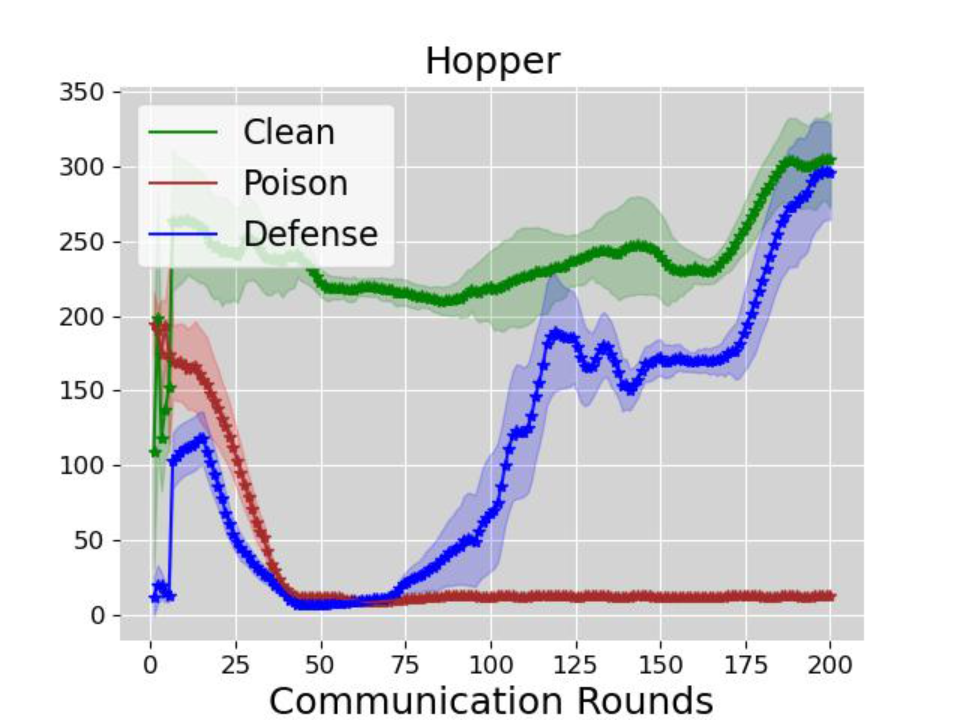}
    \includegraphics[height=1.23in]{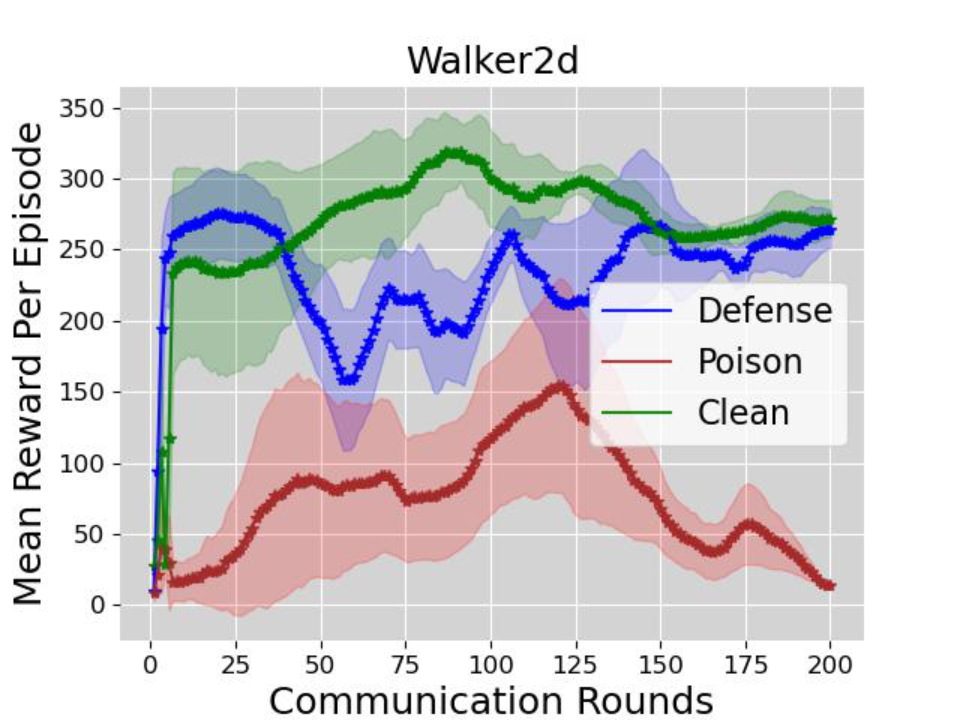}
    \includegraphics[height=1.23in]{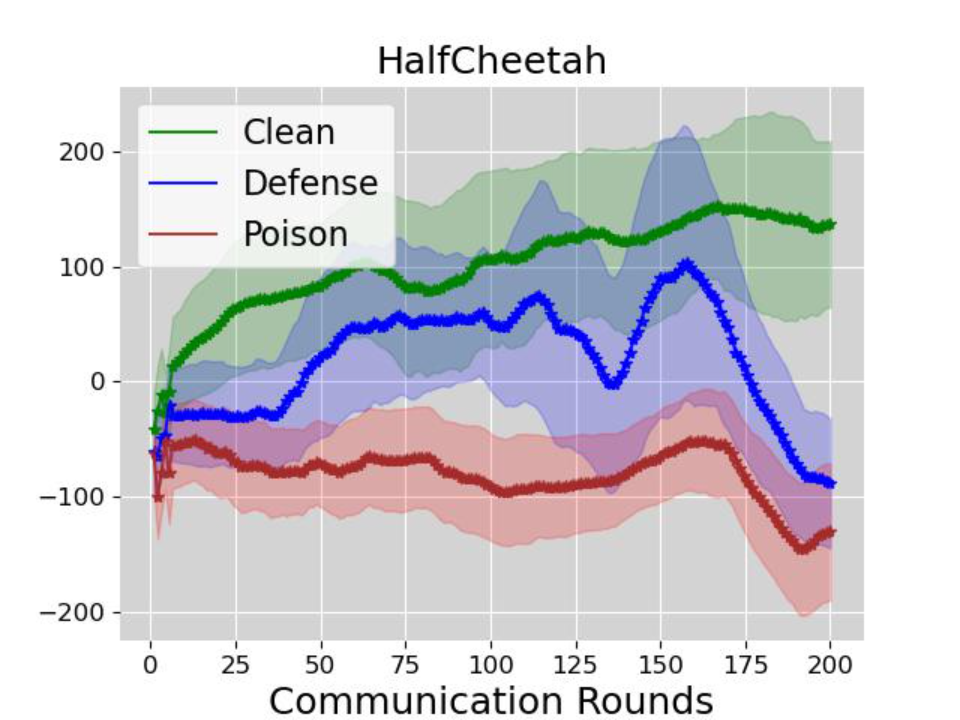}
    \caption{We train a two-agent FRL system and report results of the clean baseline, the poisoned model (budget $\eps$ = 1), and the poisoned model with a defense mechanism.}
    \label{fig:vpg-df}
\end{figure}

\subsection*{PPO-specific FRL poisoning.} 
\label{subsec:PPOpoison}
As an actor-critic algorithm, when we fit single-agent PPO into a federated framework, we assume that besides the actor model, the critic model should also be updated from individual agents to the server, then aggregated by the server and finally broadcast to the local agents at each federated round $p$. Denote the aggregated actor as $\theta_{(0)}^{p}$ and the aggregated critic as $\omega_{(0)}^{p}$. To specify the agents' local initialize function $\mathcal{A}^{up}(\cdot)$ (Eq.(\ref{eq:gen-i-init})) and the server's aggregation function $\mathcal{A}^{agg}$ (Eq.(\ref{eq:gen-agg})), we take a conventional paradigm in FL:
where $ \mathcal{A}^{up}$ assigns the server's broadcast model to local agents' initial model in the next round of local training~\cite{bhagoji2019analyzing, zhang2019poisoning, bagdasaryan2020backdoor}; $\mathcal{A}^{agg}$ aggregates the local models by adding the averaged local model update to the server's model~\cite{bhagoji2019analyzing, bagdasaryan2020backdoor}. By substituting $\theta_{(i)}^{p,0} = \theta_{(0)}^{p-1}$, $A^{agg}$ is equivalent to assigning the averaged local model as the server's model~\cite{bhagoji2019analyzing}.

We set the poison cost as $D(\mathcal{O}^{p,q}, \widehat{\mathcal{O}}^{p,q})\! =\! \|\mathcal{O}^{p,q} - \widehat{\mathcal{O}}^{p,q}\|_2$, and thereby propose the PPO-specific Problem as \ref{opt-PPO}:
where $\ \widehat{\cdot}\ $ denotes the poisoned variables induced by $\widehat{\mathcal{O}}^{p,q}$. The constraints interpretation is similar to that in Section \ref{subsec:PF1}, 
except that all the equations related to $\omega$ characterize the initialization and local training for the critic, while those related to $\theta$ are for the actor. The constraints are summarized in Table (\ref{tab:eq-ppo}).

\begin{table}[H]
\caption{Constraints of Problem (\ref{opt-PPO}).}
\label{tab:eq-ppo}
\vskip 0.15in
\begin{center}
\begin{small}
\begin{sc}
\begin{tabular}{llll}
\toprule
 Party& Model & initialization&  train\\
\midrule
clean agents & actor & Eq.(\ref{eq:ppo-ac-init}) &  Eq.(\ref{eq:ppo-ac-train})\\
 &critic & Eq.(\ref{eq:ppo-cr-init}) & Eq.(\ref{eq:ppo-cr-train})\\
\midrule
attacker &actor & Eq.(\ref{eq:ppo-ac-init}) & Eq.(\ref{eq:ppo-ac-poison})\\
&critic & Eq.(\ref{eq:ppo-cr-init}) & Eq.(\ref{eq:ppo-cr-poison})\\
\bottomrule
\end{tabular}
\end{sc}
\end{small}
\end{center}
\vskip -0.1in
\end{table}

\section*{Appendix C: A Defense Mechanism}
\label{sec:defense}

\medskip
To mitigate the risk of FRL being exposed to malicious agents, below we describe a defense mechanism against FRL (Algorithm (\ref{alg:defense})) which inherits from conventional FL defense. The core idea is to assign a credit score to each agent's policy based on its performance during testing. To that end, the central server runs tests on each policy it receives from the agents and records the observations denoted by ${O_{(i)}^{p,test}}$. The server then calculates the average reward $\overline{r}_{i}^{p,test}$ for each policy by averaging the rewards in the sequence $\{r_t\}_{(i)}^{p,test}$. Finally, the server normalizes the average rewards by dividing them by the sum of all averaged rewards, resulting in a set of normalized weights ${c}_{i}^{p,q}$. These weights are used to weight the average aggregation of the policies:
The defense mechanism is outlined in Algorithm (\ref{alg:defense}). This protocol can be integrated into Algorithm (\ref{alg:poison}) and (\ref{alg:poison_2}) as the aggregation algorithm $\mathcal{A}^{agg}$.

\vspace*{-.2cm}
\begin{algorithm}[ht]
   \caption{FRL Defense Aggregation}
   \label{alg:defense}
\begin{algorithmic}[1]
    \STATE {\bfseries Input:} 
    Submitted local actors by the agents $\{\widehat{\theta}_{(i)}^{p,L}\}_{i\in\mathcal{M}}$ and  
    $\{\theta_{(i)}^{p,L}\}_{i\notin \mathcal{M}}$; Submitted local critics
    $\{\widehat{\omega}_{(i)}^{p,L}\}_{i\in\mathcal{M}}$ and
    $\{\omega_{(i)}^{p,L}\}_{i\notin \mathcal{M}}$.
    \STATE {\bfseries Output:} Aggregated actor and critic $\theta_{(0)}^p$, $\omega_{(0)}^p$.
   \FOR{$i=1$ {\bfseries to} $n$}
   \IF{$i\in \mathcal{M}$}
   \STATE Obtain $O_{(i)}^{p, test}$ by $\widehat{\theta}_{(i)}^{p,L}$
   \ELSE
   \STATE Obtain $O_{(i)}^{p,test}$ by $\theta_{(i)}^{p,L}$
   \ENDIF
   \STATE Get mean reward $\overline{r}_{(i)}^{p,test}$ from $O_{(i)}^{p,test}$.
   \ENDFOR
   \FOR{$i=1$ {\bfseries to} $n$}
   \STATE Normalize the credit ${c}_{i}^{p,q} \leftarrow \frac{\overline{r}_{i}^{p,q}}{\sum_i \overline{r}_{i}^{p,q}} $
   \ENDFOR
   \STATE Obtain $\theta_{(0)}^{p}$ and $ \omega_{(0)}^{p} $ by Eq. (\ref{eq:dfs-agg1}) and (\ref{eq:dfs-agg2}).
\end{algorithmic}
\end{algorithm}

\section*{Appendix D: Additional Experiments}
\textbf{Additional general settings for the FRL system}. We measure the attack cost by the $\ell_2$ distance between the poisoned reward and the ground-truth observed reward. During each local training step, the maximum number of steps before termination is 300 unless restricted by the environment. The learning rate is set to 0.001, and the discount parameter is set to $\gamma=0.99$.

\smallskip
\noindent
\textbf{Target policy settings in target-poisoning experiments}. The targets are $0 \in \{0, 1\}$ for CartPole, $0 \in \{0, 1, 2, 3\}$ for LunarLander, $3 \in [-3, 3]$ for InvertedPendulum, and $0 \in [-1, 1]^6$ for HalfCheetah. 

\smallskip
\noindent
\textbf{Defense Performance.} In Fig. (\ref{fig:vpg-df}), we evaluate the defense method proposed above against our poisoning attacks on various environments using the VPG algorithm and a low budget. For each model uploaded to the server, we ran it for 10 episodes and collected the mean reward per episode. We then used the normalized rewards as the weights of aggregation. The results show that while the defense method can improve the rewards of a poisoned model, there is still a significant gap compared to the clean baseline during the training process. Figure (\ref{fig:vpg-df}) suggests that the defended FRL system is often protected from poisoning attacks but cannot reach as good as clean training. Furthermore, it is still vulnerable in complex environments such as Half Cheetah. This highlights the need for designing more robust FRL mechanisms in the future.

\begin{figure}[H]
    \centering
    \includegraphics[height=1.29in]{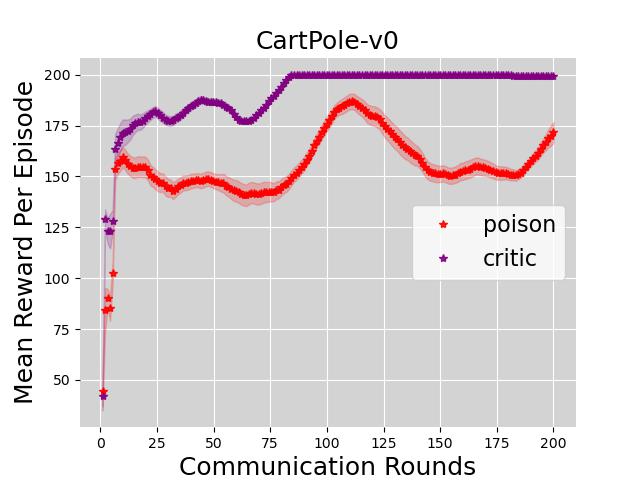}
    \caption{The two-critic PPO attack (labeled \textit{poison} in the plot) is superior compared to single-critic (labeled \textit{critic} in the plot).}
    \label{fig:critic}
\end{figure}

\noindent
\textbf{Large budget}. We evaluated the performance of our poisoning with a large budget ($\eps = 100$) in Figure (\ref{fig:vpg-r100}). The results show that with a larger budget, our method is able to attack much larger systems compared with the constraints that appeared in the case of a small budget, i.e., we can attack up to 100 agents.

\begin{figure}[t]
    \centering
    \includegraphics[height=1.2in]{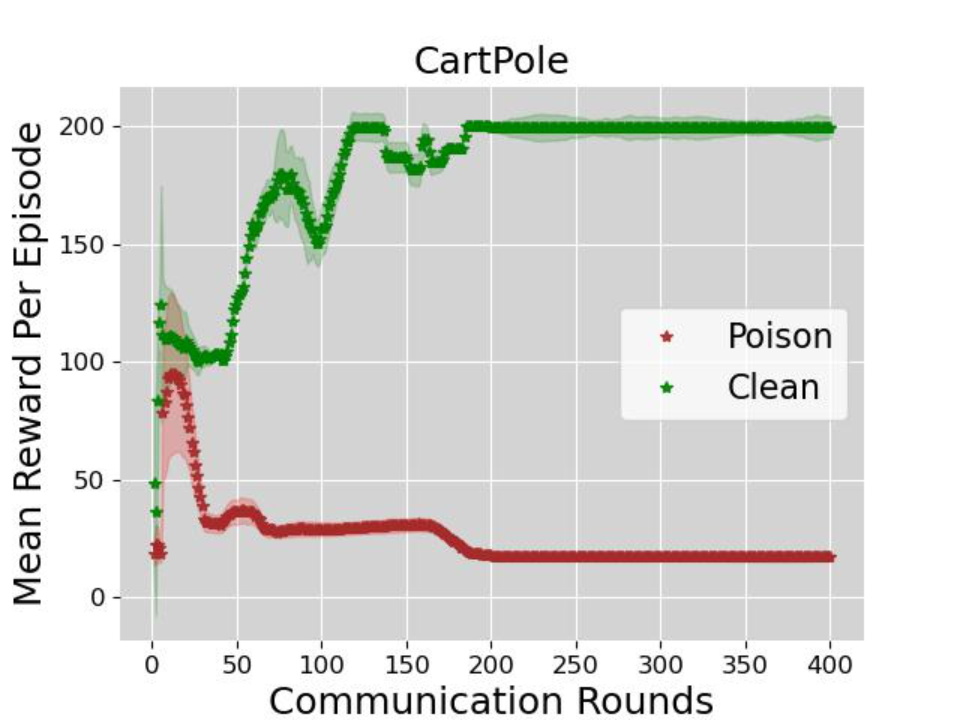}
    \includegraphics[height=1.2in]{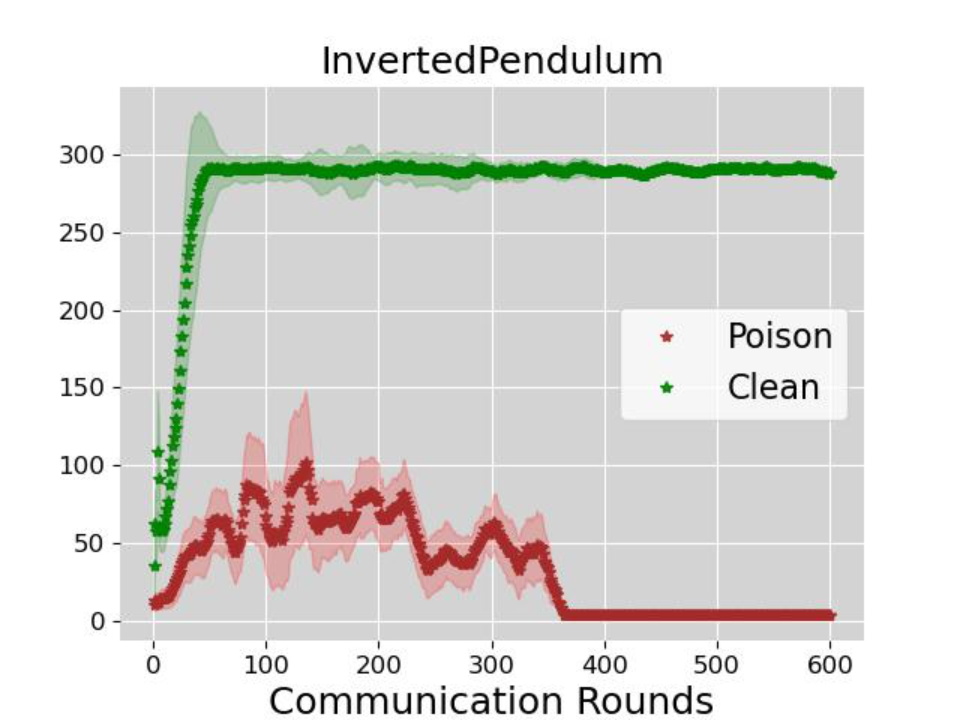}
    \includegraphics[height=1.2in]{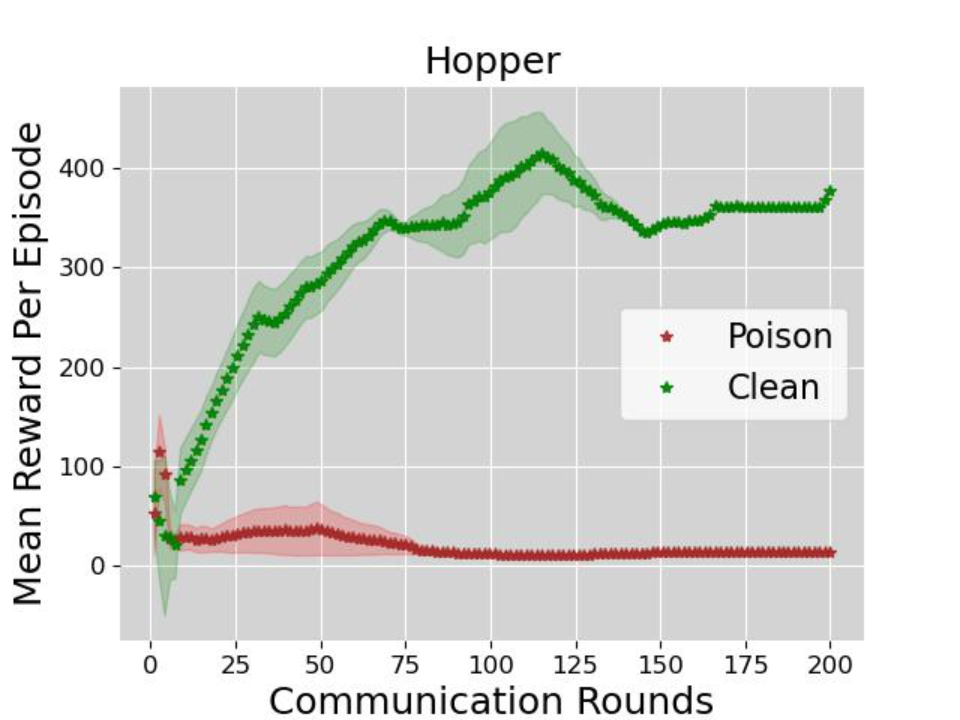}
    \includegraphics[height=1.2in]{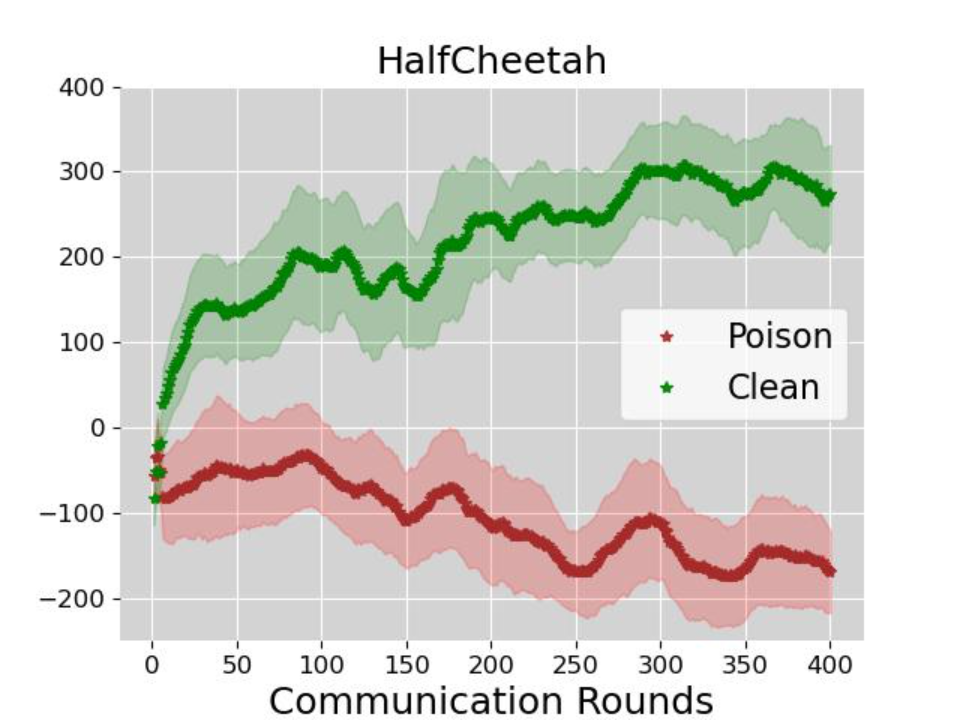}
    \caption{Contrastive performance between poisoned VPG system with high budget ($\eps = 100$) and a clean system of same agent size. One attacker can successfully poison a system composed of at least 100 agents.}
    \label{fig:vpg-r100}
\end{figure}

\smallskip
\noindent
\textbf{Attacker proportion generalization.} We have shown that the proportion of malicious agents required to poison a system is consistent regardless of the size of the system. We demonstrate this by analyzing the combination of the VPG algorithm and CartPole task with a low budget of $\eps=1$. We found that when we increase the system size from 4 agents to 100 agents, a fixed proportion of attackers can always poison the system successfully. This result is depicted in Figure (\ref{fig:vpg-gen}).

\begin{figure}[H]
    \centering
    \includegraphics[height=1.2in]{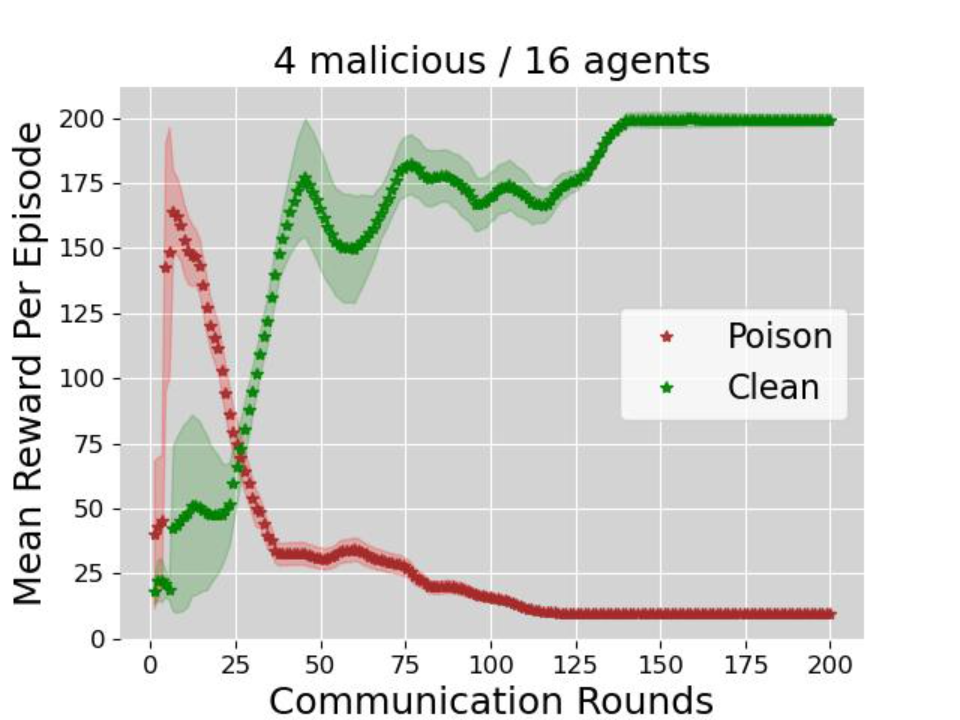}
    \includegraphics[height=1.2in]{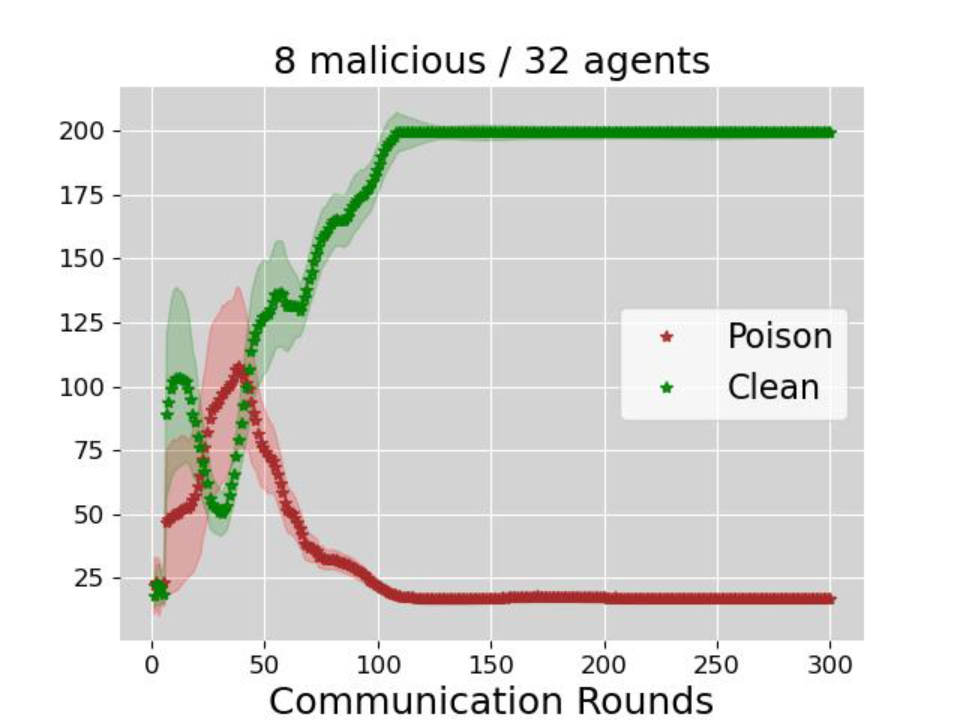}
    \includegraphics[height=1.2in]{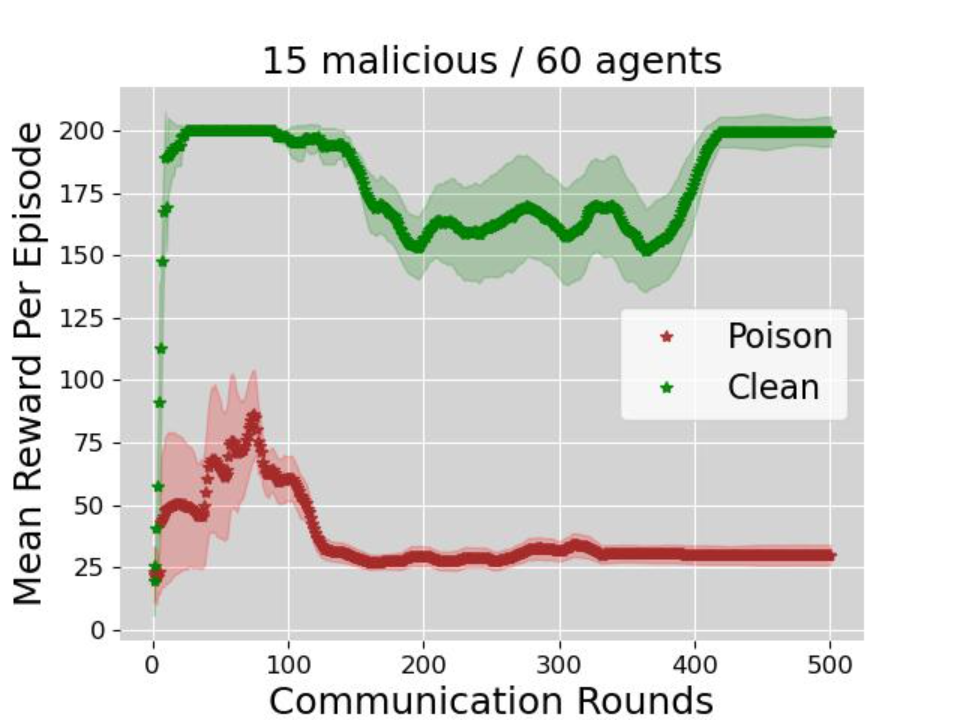}
    \includegraphics[height=1.2in]{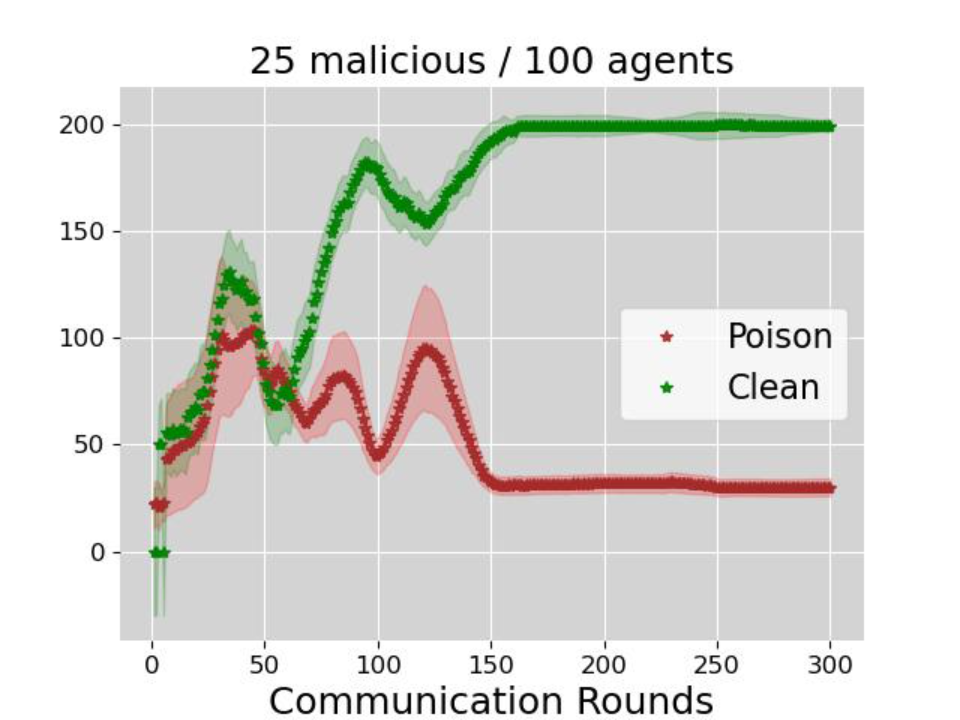}
    \caption{For the VPG algorithm and CartPole task with low budget $\eps$ = 1, while the size grows, the system can always be effectively poisoned given a fixed proportion of malicious agents.}
    \label{fig:vpg-gen}
\end{figure}

\smallskip
\noindent
\textbf{Targeted attack.} For simplicity, we choose the target policy to be a single action, whether in discrete or continuous space. Results for two discrete and two continuous environments, with varying cardinality and dimensions for the action space, are shown in Figure (\ref{fig:target}).

\smallskip
\noindent
\textbf{Single-critic attack.} For our proposed PPO method, attackers have a pair of public and private critics, among which the former is poisoned and sent to the central server, while the latter is clean and used to attack rewards. We test an alternative with only one critic that both attacks and has its rewards poisoned. To test the effectiveness of our two-critic PPO attack, we also test a single-critic approach, in which the critic poisons the rewards for both the policy and itself using its value function and gets aggregated and updated by the central server after each round. Using the same runs as presented above, we compare this alternative attack compared to ours. In Figure (\ref{fig:critic}), we can see that maintaining distinct private and public critics is superior and contributes significantly to the success of our method.

\end{document}